\documentclass{IEEEoj}
\usepackage{cite}
\usepackage{amsmath,amsthm, amssymb,amsfonts}
\usepackage[utf8]{inputenc}
\usepackage{breqn}
\usepackage{cite}
\usepackage{algorithmic}
\usepackage{graphicx}
\usepackage{textcomp}
\usepackage{xcolor}
\usepackage{subcaption}
\usepackage{multirow}
\usepackage{booktabs}
\usepackage{csquotes}
\usepackage{algorithmic}
\usepackage{graphicx,color}
\newtheorem{definition}{Definition}
\newtheorem{theorem}{Theorem}

\usepackage{textcomp}
\def\BibTeX{{\rm B\kern-.05em{\sc i\kern-.025em b}\kern-.08em
    T\kern-.1667em\lower.7ex\hbox{E}\kern-.125emX}}
\AtBeginDocument{\definecolor{ojcolor}{cmyk}{0.93,0.59,0.15,0.02}}
\begin{document}


\title{A Systems-Theoretical Formalization of Closed Systems}

\author{
    \IEEEauthorblockN{
        Niloofar Shadab\IEEEauthorrefmark{1}, Tyler Cody\IEEEauthorrefmark{1}\IEEEauthorrefmark{2}, Alejandro Salado\IEEEauthorrefmark{3}, Peter Beling\IEEEauthorrefmark{1}\IEEEauthorrefmark{2}
    }
    \IEEEauthorblockA{\IEEEauthorrefmark{1} Grado Department of Industrial and Systems Engineering, Virginia Tech}
    \IEEEauthorblockA{\IEEEauthorrefmark{2} National Security Institute, Virginia Tech}
    \IEEEauthorblockA{\IEEEauthorrefmark{3} Department of Systems and Industrial Engineering, The University of Arizona}
}
\corresp{Corresponding author: Niloofar Shadab (e-mail: nshadab@vt.edu).}

\begin{abstract}
     There is a lack of formalism for some key foundational concepts in systems engineering. One of the most recently acknowledged deficits is the inadequacy of systems engineering practices for engineering intelligent systems. In our previous works, we proposed that closed systems precepts could be used to accomplish a required paradigm shift for the systems engineering of intelligent systems. 
     However, to enable such a shift, formal foundations for closed systems precepts that expand the theory of systems engineering are needed. The concept of closure is a critical concept in the formalism underlying closed systems precepts. In this paper, 
     we provide formal, systems- and information-theoretic definitions of closure to identify and distinguish different types of closed systems. Then, we assert a mathematical framework to evaluate the subjective formation of the boundaries and constraints of such systems. Finally, we argue that engineering an intelligent system can benefit from appropriate closed and open systems paradigms on multiple levels of abstraction of the system. In the main, this framework will provide the necessary fundamentals to aid in systems engineering of intelligent systems.
\end{abstract}



\maketitle

\section{Introduction}
 
There has long been a call for a theory of Systems Engineering (SE) within the SE community with the aim of establishing SE  as a standalone engineering field capable of addressing modern engineering problems\cite{hazelrigg2022toward}\cite{kasser2011unifying}. However, there is an existing gap in concrete formalism and distinction for some fundamental concepts within the field that has led to ambiguity in some SE practices\cite{salado2021systems}. While such formalism might not have been necessary in the past, the emergence of new kinds of complex systems such as Artificial Intelligence (AI)-enabled systems has challenged traditional SE practices\cite{llinas2021motivations,tolk2011towards,mcdermott2021artificial,llinas2021systems,shadab2021shifting,smith2017cognitive}.

In our previous work, we identified potential gaps in the current SE foundations to address the unique nature of AI-enabled systems \cite{shadab2022closed}. We argued that intelligence is a relational property that can be characterized and engineered as a relation between the system and its context with both learning and intelligence properties embodied in the context regardless of the nature of the relations between them\cite{shadab2022closed}. In this situation, intelligence is no longer relegated to a component or the physical boundary of the system. Therefore, we posited that owing to this high coupling between AI-enabled systems and their environment, utilizing the concept of closure in SE is a potential path forward to build general engineered intelligence\cite{shadab2022closed}. We proposed that closed SE practices could be employed to realize the closed notion of this relational property between intelligent systems and their context. We further examined the possibility of employing closed systems precepts in an engineering framework in our later paper\cite{cody2022core}, concluding a lack of concrete definitions and formalism in the theory and practice of SE presents a barrier to applying closed system precept in engineering applications. Currently, most of the theoretical foundations in both systems theory and the theory of SE are bounded to the open systems precepts (i.e., inputs-outputs relations) \cite{dag2000introduction}. Although the concept of closed systems is being utilized in limited applications in SE, there is little to no theoretical basis for these practices, making SE activities based on closed system precepts prone to interpretation and over-abstraction. As we have identified at least one domain that can benefit from closed systems precepts (AI-enabled systems), the need for clear definitions and formalism becomes increasingly important in the field of SE. This paper revisits the concept of closure in SE, aiming to formalize, define, and evaluate this concept as the first step towards employing closed systems precepts for intelligent systems. 

As mentioned earlier, closure has been vaguely applied in SE with limited underlying formal foundations \cite{hutchison2018framework,di2018closed}. Various types of closure have been introduced in systems theory literature, including functional closure, organizational closure, operational closure, and informational closure, among others\cite{zeleny1981autopoiesis,varela1974autopoiesis,bednarz1988autopoiesis,bertschinger2006information,bertschinger2008autonomy, mora2012closure}. As a starting point, closure can be understood as a property of a system that makes the system closed, and a closed system is defined as one that does not exchange energy, information, or matter through its boundaries\cite{bertalanffy1968bertalanffy}. (This concept will be revisited in detail later in the paper.) However, there is little to no formal framework to describe the relationships and differences between each type of closure, and many of the closure types lack formal systems-theoretic definitions that distinguish them from the other types of closure; in fact, on many occasions, these terms are used interchangeably, which can cause confusion in the application of each type of closure\cite{goodenough2002concept}\cite{di2018closed}. In this paper, we develop formal systems-theoretic definitions for two types of closure, functional and informational closure, in systems and compare them in terms of system's characteristics and the relations between systems and their environment. Among all types of closure, functional and informational closure were selected as we believe they were more relevant to the systems engineering practices. We utilize a mathematical definition of functional dependency, information theory, and the systems-theoretic foundations for open and closed systems to produce formalism for functional and informational closure. Then, we determine the conditions and constraints to meaningfully use each of the two types of closure in systems. Our aim is to elaborate on the relations between these types of closure to determine their applications at different levels of abstractions for systems. Throughout this paper, we will use the terms \textit{closure} and \textit{closedness} interchangeably.

\section{Background in Information Theory} 
Before delving into the formalism, we will provide a brief introduction to the concepts and mathematics of information theory.

Information theory establishes a relation between information and uncertainty, where information is inversely proportional to uncertainty. It enables the receiver of information to make more accurate predictions than chance\cite{logan2012information}\cite{adami2016information}. 
Information of a system depends on the observer's degree of freedom to measure the number of unknown states of the system. Entropy is a crucial concept in information theory which measures the uncertainties of a random variable\cite{adami2016information}. In information theory, the entropy of a random variable calculates the average level of information regarding the possible outcomes of that variable\cite{von2004information}. Entropy connects information with probability and uncertainty of a random variable. We can interpret this connection as follows: \textit{the lower the entropy of a system, the more information we possess about the possible future states of the system}. Entropy can be calculated as follows:

\begin{equation}
\label{eq:eq3}
   H(X) = - \sum_{x\in{X}} p(x)\log p(x)
\end{equation}

Where $X$ is a set of random variables, and $p(x)$ is the probability of occurrence of element $x$ in set $X$. Equation \ref{eq:eq3} represents information entropy, which was introduced by Shannon to compute the information transmitted from a source to a receiver through an information channel that the receiver can identify\cite{shannon2001mathematical}. 


Building up from Equation \ref{eq:eq3}, one can define conditional and joint entropy of two discrete variables, which can be derived as follows (a similar process can be applied to continuous random variables) \cite{cover1991entropy}:

\begin{equation}
\label{eq:joint-entrpy}
   H(Y,X) = - \sum_{x\in{X},y\in{Y}} p(x,y)\log p(x,y)
\end{equation}
\begin{equation}
\label{eq:cond-entrpy}
   H(Y|X) = - \sum_{x\in{X},y\in{Y}} p(x,y)\log p(x|y)
\end{equation}

Joint entropy indicates the amount of information needed to determine the value of two discrete variables. Conditional information depicts the amount of additional information needed to determine the value of a random variable given the knowledge of the value of the other random variable. If we expand Equations \ref{eq:joint-entrpy} and \ref{eq:cond-entrpy} using simple algebra and the chain rule, we can achieve the following relation between joint entropy and the conditional entropy: 

\begin{equation}
\label{eq:eq6}
\begin{split}
   H(Y|X) = H(Y,X) - H(X)\\
   H(X|Y) = H(Y,X) - H(Y)
\end{split}
\end{equation}

In probability theory, we have $p(x,y) = p(y,x)$. As a result, we can deduce from Equation \ref{eq:joint-entrpy} that the joint entropy of $X$ and $Y$ is the same as the joint entropy of $Y$ and $X$. Using $H(Y,X) = H(X,Y)$, and simple algebra from Equation \ref{eq:eq6}, we can derive the following relation:
\begin{equation}
\label{eq:eq7}
   H(Y|X) = H(X|Y) + H(Y) - H(X)
\end{equation}

\subsection{Mutual Information}
The concept of \textit{Mutual Information} is also relevant for this paper. Mutual information captures the dependency between two random variables $X$ and $Y$. It represents the amount of uncertainty that is common to both $X$ and $Y$. Therefore, by observing one random variable, the uncertainty that is mutual with the other random variable will be resolved. In other words, mutual information determines the amount of information one can get from random variable $X$ by observing the other random variable $Y$. This information is jointly distributed according to the joint probability of $X$ and $Y$.    
The formula for mutual information is given as follows\cite{cover1991entropy}:
\begin{equation}
\label{eq:mutual-info}
   I(X;Y) = \sum_{x\in{X},y\in{Y}}p(x,y)\log\frac{p(x,y)}{p(x)p(y)}
\end{equation}

Based on the definition, mutual information is always non-negative\footnote{This property can be proved mathematically using Jensen's inequality and relative entropy. For more information please check \cite{learned2013entropy}}\cite{learned2013entropy}. Using Equation \ref{eq:cond-entrpy} and simple algebra, Equation \ref{eq:mutual-info} can be written in terms of entropy as follows:
\begin{equation}
\label{eq:mutual-info-entropy}
\begin{split}
     I(X;Y) = H(X) - H(X|Y) \\
     = H(X) + H(Y)- H(X,Y)
\end{split}
\end{equation}

Conditional mutual information can be also defined when we have three random variables and have joint distribution $p(x,y,z)$. It is a measurement of how much uncertainty is shared between $X$, and $Y$ but not with $Z$. It can be defined as follows\cite{vu2014entropy}: 
\begin{equation}
\label{eq:cond-mutual-info}
\begin{split}
     I(X;Y|Z) = - \sum_{x,y,z}p(x,y,z)\log\frac{p(x,y|z)}{p(x|z)p(y|z)}
\end{split}
\end{equation}

Equation \ref{eq:cond-mutual-info} also can be written as follows:

\begin{equation}
\label{eq:cond-mutual-info-entropy}
\begin{split}
    I(X;Y|Z) = H(X|Z) - H(X|YZ)\\
    = H(XZ)+H(YZ)-H(XYZ)-H(Z)
\end{split}
\end{equation}

\section{Background on Closed vs Open Systems In Systems Theory}\label{background}
In this section, we provide a brief introduction to the foundations of closed and open systems in systems theory to establish the background required for the formalism of closed systems precepts.

In systems theory, open and closed systems precepts are foundational precepts. Early general systems theorists and biological systems theorists defined open systems as those that have external interactions, with a boundary between the internal and the external, allowing interactions across the boundary \cite{von2010general}. This definition captures richness in the exchange of matter, energy, and information between a system and its environment \cite{mele2010brief}. A system with no external interactions is referred to as a closed system. Using a modeling framework that studies the structure, behavior, and properties of the systems in terms of relationships\cite{bertalanffy1968bertalanffy}, a closed system description can become as a special form of an open system: one whose input and output sizes are assumed to equal zero \cite{wymore2018model}.

In biology and natural sciences, an open system is a system whose border is permeable by matter and energy, while a closed system is only permeable by energy \cite{kast1972general}. In control theory, closed systems are open systems where the input is composed of feedback to adjust the output \cite{nise2020control}. Generally, \emph{closedness} is primarily used to describe the nature of open systems' boundaries, as in biology and natural sciences, or the use of feedback to adjust interactions, as in control theory. In systems theory, as discussed, it is reflected by the absence of any input-output relations in systems. The definition of a closed system is limited to the abstract notions of the absence of inputs and outputs. However, some attempts have been made to describe open and closed systems using a set-representation of main relations on the components, behaviors, functions, inputs and outputs of the system\cite{wang2015towards}. An open system in this context can be described as a 7-tuple $S = (C, B, R^c, R^b, R^f, R^o, R^i)$, where $C$ is a finite set of components, $B$ is a finite set of behaviors, $R^c$ is a finite set of component relations, $R^b$ is a finite set of behavioral relations, and $R^f$ is a finite set of functional relations, where a functional relation is defined between components and behaviors within the system. $R^o$ and $R^i$ are sets of finite output and input relations between external systems and the system of interests, respectively. 
Consequently, a closed system is a special case of an open system that does not include $R^o, R^i$. Therefore, it is represented as 5-tuple $S = (C, B, R^c, R^b, R^f)$ \cite{wang2015towards}. In this paper, however, we aim to be consistent with a systems-theoretical framework for defining systems which is based on the relations on system's inputs-outputs. 


In summary, the general systems-theoretical definition of a closed system in systems theory literature refers to a system that has no inputs or outputs. This definition can be formally shown in Definition \ref{def0}:

\begin{definition}[\textbf{Systems-theoretical Closed System}]
A system that has no input set, $\mathcal{X}$, and output set, $\mathcal{Y}$, from/to its environment.
\[\mathcal{X}=\mathcal{Y}=\emptyset\]
\label{def0}
\end{definition}

\noindent From an objective perspective, it is arguable that only the entire universe might satisfy the condition in Definition \ref{def0}. In fact, closed systems were deemed by early general systems theorists as nonexistent \cite{von2010general}. However, closedness can still be used as a relaxation to support engineering work\cite{estefan2007survey}. Specifically, a modeler could choose to ignore the existence of inputs and outputs, thereby assuming a closed system. This approximation has proven to be useful in several engineering applications, such as thermodynamics. In this paper, we posit that functional closure and informational closure are two potential paths to enable such relaxation of Definition \ref{def0} by formulating a direction on how to ignore inputs and outputs set to model a closed system. These paths provide a framework to define closed systems in specific contexts, enabling modelers and systems engineers to make simplifying assumptions and develop accurate models.

With this background on systems theory, we will provide systems-theoretic definitions of the terminologies that we use for our formalism in Section \ref{TERM}.

\section{Elaboration of Terminologies:} \label{TERM}

Noting that system boundaries are subjective and fluid, with no restriction as to the relationships between the elements that may fit within the boundaries \cite{salado2021assessment}, we begin by defining a system of interest denoted by ${S^0}$. $S^0$ is a system that will be engineered for a specific purpose. With respect to an ${S^0}$, we define the following systems:
\begin{itemize}
    \item Environment, denoted by $E$: It is a non-empty system that consists of everything outside of $S^0$.
    \item Context system, denoted by $S^C$: It is a system that consists of both ${S^0}$ and a non-empty part of $E$, which we call Inner Environment and denote by $E^I$. So, $S^C = {S^0} \cup {E^I}$.
    \item Inner Environment, denoted by $E^I$: As per the previous definition, it is a non-empty system that consists of the complement of ${S^0}$ with respect to $S^C$.
    \item Outer environment, denoted by $E^O$: It is a system that consists of the complement of $S^C$ with respect to $E$.
    \item Universe, denoted by $U$: It is a non-empty system that consists of the entire environment $E$, and the system of interest $S^0$.
\end{itemize}

The terminology used throughout this paper is depicted in Figure \ref{fig:mesh1}, where the boundary of $S^C$ is represented by a blue circle and the curved orange arrows depict the coupling between ${S^0}$ and $E^I$. In our definition, we partition the environment into two sets, ${E} = E^O \cup {E^I}$ wher $E^O \cap {E^I} = \emptyset$.

\begin{figure}[h]
    \centering
    \includegraphics[scale=0.6]{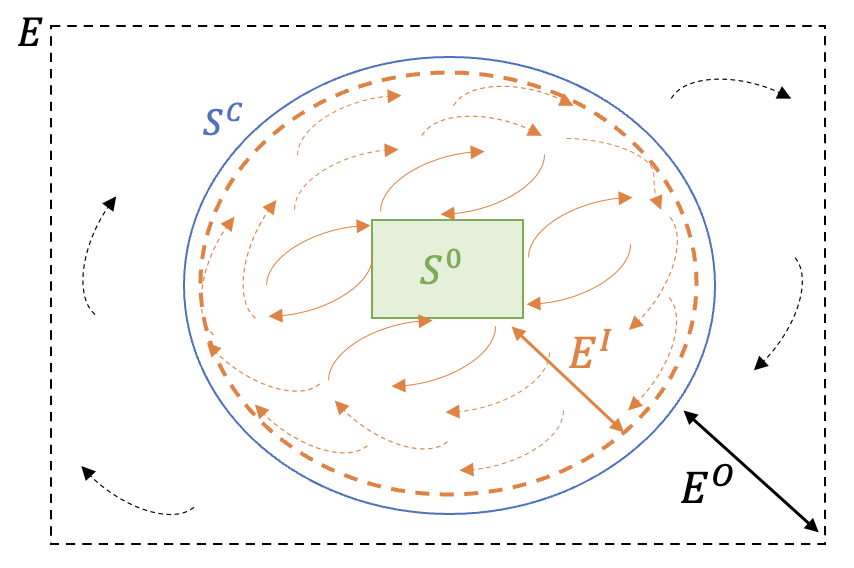}
    \caption{Interactions and relations between systems in a closed system setting are depicted. $E^O$ is the environment outside of the Context system. The context system, $S^C$, includes the system of interest, ${S^0}$, and a portion of the overall environment, the inner environment, $E^I$}
    \label{fig:mesh1}
\end{figure}

To maintain consistency with systems theory, we will use set-theoretic formalism to define system introduced in this paper. A set representation of the systems is based on the relation of the set of systems objects given the objects indices. Mathematically, we can represent this set as follows: $S \subseteq \times \{V_i |i \in I\}$ where $V_i$ denotes the $i^{th}$ element of the set. In this context, let $S$ be an input-output system $S \subseteq \times \{\mathcal{X},\mathcal{Y}\}$ and $\Bar{S}$ the component sets of set $S$, i.e., $\{\mathcal{X}, \mathcal{Y}\}$. Consequently, we can define system of interest, Context system, and Universe, $S^0, S^C, U$, as follows:

\begin{equation}
\begin{split}
    S^0: S^0 \subseteq \times \{\mathcal{X}^0, \mathcal{Y}^0\} \\
    U: \mathcal{X}^u = \mathcal{Y}^u = \emptyset\\
    S^C: S^C \subseteq \times \{\mathcal{X}, \mathcal{Y}\} \\
    \label{three-systems-eq}
\end{split}
\end{equation}

Having defined $S^0, S^C, U$, we can formulate the environment with respect to these three defined systems. As it is shown in Figure \ref{fig:mesh1}, the entire environment consists of everything outside $S^0$, therefore, we can define $E$ as follows:

\begin{equation}
\begin{split}
   E: E \subseteq \times \{\mathcal{X}^E, \mathcal{Y}^E\}\\
    \text{Where, }\Bar{E} = \Bar{U} \setminus \Bar{S^0} 
    \label{environment-eq}
\end{split}
\end{equation}

Now that we formulated the entire environment, $E$, in Equation \ref{environment-eq}, Inner environment, $E^I$, and Outer environment, $E^O$, will be presented as a subset of $E$:

\begin{equation}
\begin{split}
   E^I: E^I \subseteq \times \{\Bar{E} \cap \Bar{S}^C\} \\
   E^0: E^0 \subseteq \times \{\Bar{E} \setminus \Bar{S}^C\}
    \label{inner-outer-envo-eq}
\end{split}
\end{equation}


Having defined different systems' boundaries, we also provide mathematical definitions for the system's terminologies that we will be using in the formalism process of closed systems. These terminologies are defined in a systems-theoretical framework. The list below contains these terminologies and their mathematical representation/definitions: 

\begin{itemize}
    \item \textbf{Functional System}: $S$ is functional system if $S$ executes a function, i.e, with every input $x \in \mathcal{X}$, there is associated a single output: $S: \mathcal{X} \rightarrow \mathcal{Y}$

     In line with abstract systems theory \cite{mesarovic2006complex}, we assume all functional systems to be surjective, unless otherwise specified. Also, note that, by definition, every open system will be a functional system.
     \item \textbf{Functional Dependency}: Functional dependency is defined based on the relationship between a dependent set $Y$ and an independent set $X$ through a function $f(X)$\cite{laurini1992fundamentals}. We can express this relationship as $X \rightarrow Y$, meaning that one set (X) can accurately determine the value of the other set (Y). In a surjective function, there is always one dependent set or variable and one independent set or variable, making functional dependency a sufficient condition for a function to be surjective.  Therefore, functional dependency can be also defined as: 

\begin{equation}
    \forall y \in Y, \text{ }\exists x \in X \quad \text{s.t.} \quad S(x) = y
    \label{func-depen}
\end{equation}
    \item \textbf{Set of System's Functions}: $F$ is a set of system's function if, $F$ is defined as $F \subseteq \times \{F_i|i\in I\} \text{ such that } F:\mathcal{X} \rightarrow \mathcal{Y}$, i.e., $F$ is a mapping between system's inputs and outputs\cite{mesarovic1975general}.
    \item \textbf{System's Behavior}: System's behavior can be described in reference to the system's response to any stimulus from its environment, $x \in \mathcal{X}$ \cite{mesarovic1972conceptual}. For any stimulus $x \in \mathcal{X}$, and its associated decision problem; $\Delta(x)$, there is a solution $d(x) \in D$, where $D$ is the decision set. Therefore, system's behavior is defined as a mapping of system's decision set $D$ to the system's outputs $\mathcal{Y}$. This mapping can be shown as: $Q: D \rightarrow \mathcal{Y}.$
    \item \textbf{System's States}: State of the system represents the history of the system's behavior. The state at any point in time can be represented as an equivalence class generated by the equivalence relation defined on the system's past behaviors\cite{mesarovic1975general}.
\end{itemize}

To investigate how and under what conditions the context system may be conceived or relaxed as a closed system, we explore two potential paths: functional closure and informational closure. We examine the requirements and implications of each approach and discuss their applicability to real-world SE problems. Ultimately, our goal is to provide insights and guidance for system engineers seeking to model closed systems and understand the implications of such models.
 
\section{Functional Closure} \label{functional}
Given the terminologies introduced and formalized in Section \ref{TERM}, this paper now explores the concept of functional closure. We provide a formal definition for a functionally closed system. With the functional closure concept, instead of having a systems-theoretical closed systems, we have a relaxation of a closed system that can be considered closed only from the functional point of view.


Based on the systems-theoretic definition of functional system, and the condition of functional dependency, a system can be considered to be functionally closed with respect to its environment when the system's outputs are independent of the inputs received from the environment and the behavior of the environment is independent of the inputs it receives from the system. Mathematically, this can be represented as $\times \{\mathcal{X}, \mathcal{Y}\} = \emptyset$, indicating that there is no mapping between the system's inputs and outputs. This condition is consistent with the definition of a systems-theoretic closed system, which also requires $\mathcal{X} = \mathcal{Y} = \emptyset$, indicating no mapping between inputs and outputs. Therefore, it can be concluded that a systems-theoretic closed system, which has an empty set of inputs and outputs, is also a functionally closed system, as it implies an empty set of mappings between inputs and outputs, indicating functional independence of the system from its environment. 

However, this definition may not have any practical or formal systems-theoretical implications. In systems theory, systems are defined as relations on their inputs and outputs. Therefore, the absence of any mappings between inputs and outputs in functionally closed systems prevents us from defining such systems in a systems-theoretical formulation, as expressing relations of an empty set has no meaning. To overcome this complication, we can define a condition where a system is conceived as functionally closed if its set of functions can be defined by a minimal set of inputs and outputs. By minimal set, we suggest that any changes in the inputs set will not change the function of the system, i.e., the system's outputs set. This allows us to define a functionally closed system in a practical and meaningful way, as it captures the concept of functional independence from inputs while avoiding the issues associated with expressing relations on an empty set. We formally demonstrate functional closure using the minimal set constraint as follows:

\begin{equation}
\begin{split}
& \text{Given a functional system, $S$}:X \to Y \\
& \text{where}\quad \forall y \in Y, \text{  } \exists x \in X, \quad \text{s.t.} \quad S(x) = y \\ 
& \text{Given a functional system, $S'$}:X' \to Y \text{ where } X \subset X' \\
& \text{If}\quad \forall y \in Y\text{, } \not \exists x' \in X', \quad \text{s.t.} \quad S'(x') = y \text{ } \wedge \text{ } x' \not \in X \\
& \Rightarrow S' \text{ is functionally the same as S} \\
& \Rightarrow S \text{ is functionally independent from additional inputs $x'$} \\
& \Rightarrow S \text{ is functionally closed from } S' \\
\label{lemma0}
\end{split}
\end{equation}

Based on Functional Closure formulation in Equation \ref{lemma0}, we deduce that the outputs set of system $S$ is not dependant on any additional inputs beyond the minimal inputs set from System $S'$. Therefore, we can outline the conclusion from Equation \ref{lemma0} as follows: 

$S$ is said to be functionally closed from $S'$ if and only if:
\begin{itemize}
    \item There exists a minimal set $M$ of inputs and outputs, where $M \subseteq \times \{\mathcal{X}, \mathcal{Y}\}$, such that $S$ is functionally dependent only on $M$, i.e., $S \subseteq \times \{\mathcal{X}_M, \mathcal{Y}_M\}$
    \item There are no additional inputs beyond $M$ that can influence the behavior of the system, $S$.
    \item The state of $S'$ is not affected by the outputs of $S$.
\end{itemize}

When these conditions are present, functional independence of the system from inputs set and outputs set can be relaxed as $\mathcal{X}=\mathcal{Y}=\emptyset$. This relaxation is due to the independence of the system's functions from changes in inputs set, $\mathcal{X}$, and hence, existence of a constant set of outputs, $\mathcal{Y}$. A functionally closed system can be represented via the set of system's functions (which already captures the minimal set of input-output relation). This representation allows us to ignore $\mathcal{X}$ and $\mathcal{Y}$ in the representation of such system. Therefore, a functionally closed system can be a relaxation of the systems-theoretic closed system (Definition \ref{def0}). To prove that functional closure is a relaxation of the systems-theoretic closure, we will have:

\begin{equation}
\begin{split}
& F \subseteq \times \{F_i |i \in I\} \quad \text{s.t} \quad F:  \mathcal{X}_m \rightarrow \mathcal{Y}_m \\ 
& \text{Where:} \quad \mathcal{X}_m \subseteq M \quad \mathcal{Y}_m \subseteq M \\
& \text{If:} \quad \exists X_{E}, \quad \exists Y_E, \quad \text{s.t} \quad \mathcal{X}_m \subseteq X_{E}, \quad \mathcal{Y}_m \subseteq Y_{E} :\\
& \text{ We have: } F: X_{E} \rightarrow \mathcal{Y}_m \\
& \text{If: S is defined as } S \subseteq \times \{F_i |i \in I\} \\ 
& \text{Let: } \mathcal{X} = X_{E}\setminus M \text{ and } \mathcal{Y} = Y_{E}\setminus M \\
& \text{We have: } S \not \subset \times \{\mathcal{X}, \mathcal{Y}\} \\ 
\end{split}
\end{equation}

Where F is the set of system's functions. When the conditions stated above are present, functional independence of the system from inputs set and outputs set can be relaxed as $\mathcal{X} =  \mathcal{Y} = \emptyset$, which is the same condition for a systems-theoretic closed system. Therefore, a functionally closed system can be considered a relaxation of the systems-theoretic closed system because it is a less restrictive condition that still allows for some interaction with the environment.

Extending Equation \ref{lemma0} into the paper's terminologies, we posit that by having functional independence between the context system and its environment, the context system achieves functional closure. It means that (1) there is no inputs set from the outer environment, $E^O$, that influences the behaviors/functions of the context system, $S^C$, and therefore any of its internal functions, and (2) the state of the outer environment, $E^O$, is not influenced by the outputs of the context system, $S^C$.  Given the interpretation of functional closure for a system in Equation \ref{lemma0}, a functionally closed context system; $S^C$; can be formally defined as follows:

\begin{definition}[\textbf{Functionally Closed Context System}] 
A functional context system, $S^C$, is functionally closed from its outer environment, $E^O$, if and only if, \\
1) There exists a minimal set of inputs and outputs, , $M$, such that $S^C$ is functionally dependent on $M$. 
This condition can be shown as: $S^C \subseteq \times \{\mathcal{X}_M,\mathcal{Y}_M\}$, and \\
2) There are no additional inputs from $E^O$ beyond $M$ that can influence the behavior of $S^C$. and \\
3) There are no additional outputs from $S^C$ beyond $M$ that can affect the behavior of $E^O$. \\
Mathematically, the second and third conditions can be shown as follows:

\[\text{Given:}\quad S^C: \mathcal{X}_M \rightarrow \mathcal{Y}_M \text{ } \And \text{ } E^O: \mathcal{X^O} \rightarrow \mathcal{Y^O}\] 
\[\text{Where:}\quad y \in \mathcal{Y}_M \text{ } \And x \in \mathcal{X}_M\] 
\[\text{From Eq \ref{three-systems-eq}, and Eq \ref{inner-outer-envo-eq}, we know: }  \mathcal{Y^O} \rightarrow \mathcal{Y}_M \wedge \mathcal{X}_M \subseteq \mathcal{Y^O} \]
\[\text{If:}\quad  x': x' \in \mathcal{Y^O} \wedge x' \not \in \mathcal{X}_M\]
\[\forall y \in \mathcal{Y}_M,\text{ } \not \exists x' \in Y^O, \quad \text{s.t.} \quad S^C(x') = y \quad\] \\

\label{functionalClosure}
\end{definition}

This definition of a functionally closed system in Definition \ref{functionalClosure} shows that it is a relaxation of the systems-theoretic definition of a closed system. It allows for the context system to have interactions with the environment, but these interactions must not affect the behavior of the context system or its outputs. 

It should be noted that functional closure derived from Definition \ref{functionalClosure} is fundamentally different from a closed control system where inputs and outputs are coupled together through feedback loops. The thorough comparison between these two types of systems is discussed in our previous work \cite{shadab2022closed}. 


\subsection{Functional Closure Constraint For Engineering Purposes}
Figure \ref{image002} shows the translation of system of interest's relations with its environment through functional closure. In the top diagram, we have $Item 3 = f(Item 2, Item 5)$. In the bottom diagram, we have $Item 3 = f(Item 2)$. Assume that system of interest's behavior is independent of the outer environment (or that its effect is negligible). Then, we can state: $Item 3 = f(Item 2, Item 5) = f(Item 2)$. Functional closure allows us to remove the arrow $Item 5$. The same story can be constructed from the perspective of $Item 6$, assuming that the SOI causes negligible impact on the outer environment.

From the perspective of system of interest, the context is functionally closed if also for the inner environment, we can ignore the impact of $Item 1$ in $Item 2 = f(Item 3, Item 1)$. Therefore, we have: $Item 2 = f(Item 3, Item 1) = f(Item 3)$. The same story would be constructed for $Item 4$.

\begin{figure}[h]
    \centering
    \includegraphics[scale=0.5]{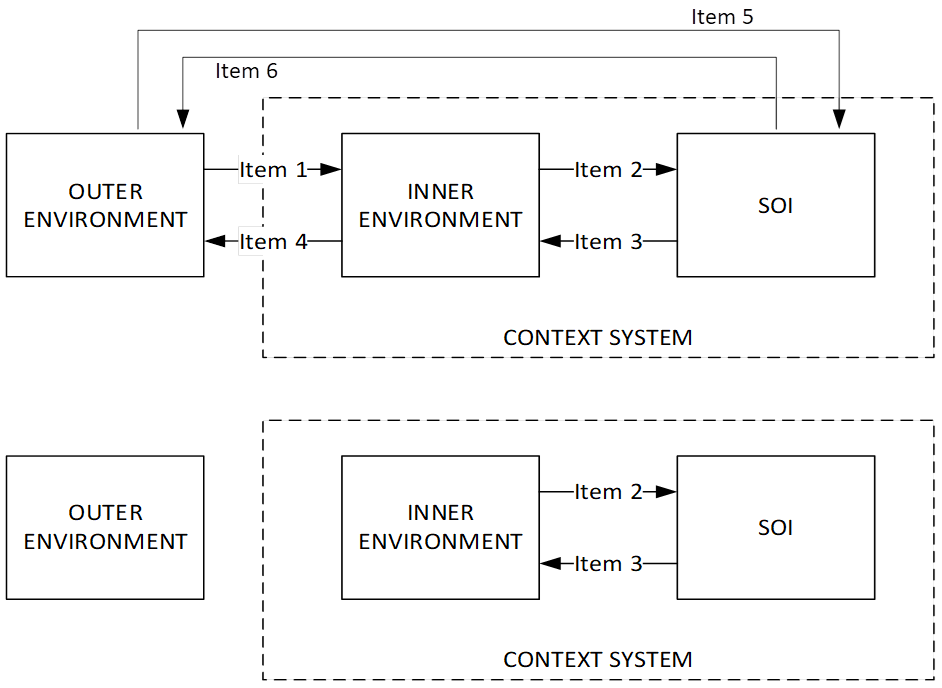}
    \caption{Top Diagram captures all inputs and outputs between $S^0, E^I$ and $E^O$. Bottom diagram shows functional closure framing of the top diagram.}
    \label{image002}
\end{figure}

Therefore, to have functional closure between $S^C$ and $E^O$, we need to attest that there is functional independence between these two systems. The absence of functional dependencies between $S^C$ and $E^O$ implies that both $S^0$, $E^I$ have functional independence from $E^O$. Since $S^0$, and $E^I$ are functionally dependent on each other (within closed system's boundary), if $E^I$ is independent from $E^O$, $S^0$ will also be independent from $E^O$ and vice versa. 




In a scenario where $E^I$ and $E^O$ can be relaxed as functionally independent, $S^C$ is functionally closed from $E^O$. 
Relaxation of functional closure may suffice in some engineering purposes. Some SE practices such as mission engineering are based on the assumption of functional independence of $E^I$ and $E^O$ and functional closure of $S^C$. However, an absence of mapping between inputs and outputs for the two partitions of the environment is not a realistic condition in many of the engineering applications. 


We demonstrated that for $S^C$ to be engineered as a functionally closed system, it either requires to encompass all of ${E}$, which is the universal system or to deal with an environment that can be divided into two functionally independent parts which is not a valid assumption for the majority of the complex environments. In functionally closed systems, we ought to create a system that is built upon its minimal set of inputs and outputs and will not change or functionally evolve over time. Therefore, a functionally closed system can be suitable for some categories of SE practices and cumbersome for the rest. The modeling benefits of this relaxation would be realized if the system is characterized at an accurate level of abstraction. We, next, examine informational closure as another path to have a relaxation of closure for engineering systems.

\section{Information Closure}
In this section, similar to Section \ref{functional}, we will elicit formalism of informational closure definitions. Then we will examine if compared to functional closure, informational closure is a better solution for a relaxation of the systems-theoretical definition of a closed system.

Informationally closed system can be achieved when there is no \textit{new} information exchanged between the context system $S^C$ and its environment $E^O$\cite{bertschinger2006information}. This definition is different from functional closure where any type of information (new or expected) cannot cross the closed system's boundary. The interpretation of this definition is that mutual information between the closed system at state $n$, $S^C_n$, and its environment, $E^O_n$, needs to be enough to ignore any new information transmission between these two systems. This closure implies that the joint information between the closed system at state $n+1$, $S^C_{n+1}$, and the environment, $E^O_n$, given the information from the closed system at state $n$, $S^C_n$, should also be zero \cite{bertschinger2006information}. It indicates that the future state of an informationally closed system should not be dependent on the conditional information of its environment given the information of the present state of the system. Consequently, the current environment does not contain any information regarding the future states of the system that has not already been present in the current state of the system; this notion is an indication of informational closure\cite{bertschinger2008autonomy}. 

According to Definition \ref{functionalClosure}, a functionally closed system is not affected by any changes in its inputs set. In contrast, in an informationally closed system, there could be changes in the inputs and outputs sets to and from the closed system, $S^C$, that can be interpreted in the form of mutual information between the closed system; $S^C$; and its environment; $E^O$. An informationally closed system considers a boundary or scope that distinguishes the external environment, $E^O$, from the internal system, $S^C$, however, in contrast to functionally closed systems, it does not necessitate that the mapping between inputs and outputs are bounded to a minimal set, but rather that the changes in the inputs set; $\mathcal{X}$; will influence the set of outputs $\mathcal{Y}$ by changing/updating the content of mutual information between the current states of the closed system and outer environment. In functionally closed systems, on the other hand, we don't consider any influence on input-output sets crossing between the closed system and environment beyond their minimal sets. To translate this property of informationally closed systems, the systems-theoretic framework that is based on the mapping relations between $\mathcal{X}$ and $\mathcal{Y}$ should be extended such that it can take into account the dynamic nature of information transmission at each state in the closed system. In contrast to functionally closed system that closure is a static property of the system, informationally closed systems are defined at system's states. Closure is a transitional property that is bounded by the previous states of the system and outer environment. Therefore, to define this type of closed system, we utilize information theory.  

As described earlier in this paper, we consider a system \textit{informationally closed} when there is no flow of new information between the environment and the system. For informationally closed systems, from the definition that we provided earlier\footnote{Informational closure can be achieved when the flow of new information between the closed system, $S^C_n$, and its environment, $E^O_n$ sets to zero} , we have  $I(S^C_{n+1}; E^O_{n}|S^C_{n}) \rightarrow 0$ \cite{bertschinger2008autonomy}. We can frame this definition as follows:\\

\begin{definition}[\textbf{Interpretation of An Informationally Closed Systems Using Information}]
A Context System that transitions through states $1, 2, …,n, n+1$; is informationally closed at state $n$ if there is no joint information between $S^C_{n+1}$ and $E^O_n|S^C_n$.
 
\label{info-info-def}
\[I(S^C_{n+1};E^O_{n}|S^C_{n}) = 0\]

\end{definition}

\newtheorem{prop}{Proposition}
\begin{prop}
If $S^C$ is informationally closed, joint information of $S^C_{n+1}, E^O_{n}, S^C_{n}$, equals to joint information between $S^C$ at state $n$ and state $n+1$:

\[I(S^C_{n+1};E^O_{n},S^C_{n}) = I(S^C_{n+1};S^C_{n})\]

\end{prop}
\begin{proof}

 \[I(X;Y)=H(X)-H(X|Y) \Longrightarrow\] 
 \[I(S^C_{n+1};E^O_{n},S^C_{n}) = H(S^C_{n+1})-H(S^C_{n+1}|S^C_{n},E^O_{n})\]

 Thus; from Equation \ref{eq:mutual-info-entropy} and Equation \ref{eq:cond-mutual-info-entropy}, and substituting X, Y, Z with $S^C_{n+1}, E^O_{n}, S^C_{n}$, we have: 
 
\[I(S^C_{n+1};E^O_{n},S^C_{n})=\] 
\[I(S^C_{n+1};,S^C_{n})+I(S^C_{n+1};E^O_{n}|S^C_{n}) \rightarrow\]
\[I(S^C_{n+1};E^O_{n},S^C_{n})=I(S^C_{n+1};,S^C_{n})\]

\end{proof}

\subsection{Functional vs Informational Closure}
To rigorously compare informational closure with functional closure, we will provide an informational interpretation of functional closure. Interpreting functional closure using the concept of joint information follows the set-theoretic formalism of having the mapping between the environment and the closed system, $\times \{\mathcal{X}, \mathcal{Y}\}$, as an empty set beyond their minimal set. The absence of any mappings between outputs set and changes in inputs set indicates that the closed system was built upon the minimal information from the environment. No information from the environment would change the output of the system. If the information from outer environment enters the system at state $n$ and the output of the system is produced at the next state of the system, $n+1$, functional closure can be relaxed as if there would be no information transition between the next state of the system and the current state of outer environment. This absence of information transition from the environment to the closed systems reflects in the next state of the system such that there would be no mutual information between the next state of the system and the current state of the environment. This reflection suggests that all the information in the next state of the system will be provided by the current state of the system. Using this information-theoretic framing, functional closure for the context system; $S^C$; can be defined as:\\

\begin{definition}[\textbf{Interpretation of A Functionally Closed System Using Information}]
A Context System that transitions through states $1, 2, …,n, n+1$; is functionally closed if there is no mutual information between $E^O_n$ and $S^C_{n+1}$:
\[I(S^C_{n+1};E^O_{n}) = 0\]

\label{info-func-def}
\end{definition}

\begin{prop}
If $S^C$ is functionally closed, all the information existing in $S^C_{n+1}$ comes from $S^C_n$. Therefore, we have:

\[I(S^C_{n+1};S^C_n) = H(S^C_{n+1})\]

\end{prop}

\begin{proof}

From Equation \ref{eq:mutual-info-entropy}, the joint information of $S^C_{n+1}$ and $S^C_n$ in a functionally closed system would be: 

\[I(S^C_{n+1};S^C_n)=H(S^C_{n+1})- H(S^C_{n+1}|S^C_n)\]

As the entire information of $S^C_{n+1}$ comes from $S^C_n$, $H(S^C_{n+1}|S^C_n)$ will be zero. 

 \[H(S^C_{n+1}|S^C_n) = 0\]
 
 Thus;
 \[I(S^C_{n+1};S^C_n)=H(S^C_{n+1})\]
 
\end{proof}

In a functionally closed system, the inputs set from $E^O$ at state $n$ would not affect the outputs set of $S^C$ at state $n+1$. Having no effects of inputs change on the outputs set, we argue that there is also no joint information between the system at state $n$, $S^C_{n}$, and environment at state $n$, $E^O_n$. Definition \ref{info-func-def} could therefore be extended such that it includes the fact that there will be no input from the environment, $E^O_n$,  with new information to the system at state $n$ and no output from the system with new information to the environment at state $n$. This relation emphasizes on the static nature of functional closure.

Informational closure, on the other hand, adds a dynamic nature to functional closure where the system can be considered closed relative to its set of states and the outer environment set of states. Informational closure implies that an informationally closed system is not derived from the constraints on inputs set. It is rather derived from the ability to predict or expect such an inputs set from its environment, $E^O_n$. This implication again is related to the relationship between the concepts of information and uncertainty in systems which means the more the mutual information, the less the uncertainty related to inputs set. Intelligent systems in particular can benefit from utilizing such closure as these systems are expected to have the capability of prediction of inputs from the environment. In line with this implication, later in the paper, we will argue how informational closure can be employed as an engineering constraint to build intelligent systems.

In informational closure, there is causal dependency between the environment; $E^O_n$; and the closed system; $S^C_n$. Another major distinction between the informationally and functionally closed systems roots in the difference in assumptions about the causal (or functional) dependencies between the system and its environment. In Definition \ref{functionalClosure}, it is shown that there is no functional (causal) dependencies between the functionally closed system and its environment beyond its minimal set. In contrast, there are inputs and outputs to and from an informationally closed system which can be defined as mutual information. This input-output relation; mutual information; can be built upon various types of causal relationships between the two systems. Mutual information can have causality originated from either the closed system or its environment. The causal dependency analysis for this type of closure answers the question of ``where the information originates from". However, in this paper, we investigate \textit{``the very existence of mutual information transmitting through the closed system's boundary"}. Thus, although, understanding causal relations is necessary for engineering applications, the causal dependency of this mutual information is not the focus of this paper, and it could be considered as a potential future work.

\subsection{Informational Closure constraint for Engineering Purposes}\label{CSE}
Figure \ref{image003} shows the translation of system of interest's relations with its environment through informational closure. In the top diagram, we have the same systems as in Figure \ref{image002}. In the bottom diagram, we have an additional $Item 7$ which is the mutual information between the informationally closed system and its environment $E^O$. This mutual information is a subset of all the inputs-outputs to/from $S^C$ and $E^O$. The impact of the rest of the inputs-outputs can be ignored in $S^C$.

\begin{figure}[h]
    \centering
    \includegraphics[scale=0.44]{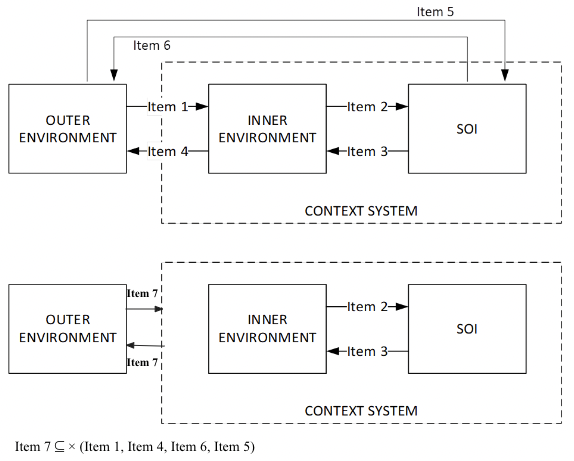}
    \caption{Top Diagram captures all inputs and outputs between $S^0, E^I$ and $E^O$. Bottom diagram shows informational closure framing of the top diagram.}
    \label{image003}
\end{figure}

So far, we provided a foundation of why mutual information between the informationally closed system and its environment should stay high enough to maintain the condition of informational closure. As a result, the next step is to define ``high enough" level of mutual information to enable realization of such a closed system in engineering applications. For this purpose, we need to find the condition for having at least a minimum level of required mutual information $I(S^C_n;E^O_n)$ between the context system; $S^C_n$; and Outer Environment; $E^O_n$; to realize an informationally closed system at state $n+1$. 

To derive such a constraint for mutual information in the context of informational closure, we will employ the entropy and information formulas while observing the bounded definition of informational closure. To do so, we assume that the informationally closed system is the sole sender and receiver of the information at different time steps as there will be no new information coming in and out of the closed system's boundaries. Accordingly, we consider the system at state $n$; $S^C_n$; as the sender of the information, and the system at state $n+1$; as the receiver of such information. From Equation \ref{eq:mutual-info-entropy}, we can write:
\begin{dmath}
    \label{eq:eq11_0}
    I(S^C_{n+1}; E^O_{n}, S^C_{n}) = H(S^C_{n+1})-H(S^C_{n+1}|S^C_{n}, E^O_n) 
\end{dmath}

If we use Equation \ref{eq:cond-mutual-info} and substitute the last entropy in Equation \ref{eq:eq11_0}, we can rewrite Equation \ref{eq:eq11_0} as follows:

\begin{dmath}
    \label{eq:eq11_1}
    I(S^C_{n+1}; E^O_{n}, S^C_{n}) = I(S^C_{n+1};S^C_n)+ H(S^C_{n+1}|S^C_{n}) + I(S^C_{n+1};E^O_n|S^C_{n}) - H(S^C_{n+1}|S^C_{n})
\end{dmath}

We can now define the mutual information between these systems and their environment $E^O_n$ using Equation \ref{eq:eq11_1} and chain rule as follows\cite{krakauer2020information}:

\begin{dmath}
    \label{eq:eq11}
    I(S^C_{n+1}; E^O_{n}, S^C_{n}) = I(S^C_{n+1}; S^C_{n}) + I(S^C_{n+1}; E^O_{n}|S^C_{n})\\
    = I(S^C_{n+1};E^O_{n}) + I(S^C_{n+1}; S^C_{n}|E^O_{n})
\end{dmath}

Equation \ref{eq:eq11} is derived using the chain rule, as well as conditional and mutual information formulas. For informationally closed systems, from Definition \ref{info-info-def}, we have  $I(S^C_{n+1}; E^O_{n}|S^C_{n}) = 0$. This condition depicts that for an informationally closed system, the amount of mutual information between $S^C_{n+1}$ and $E^O_n|S^C_n$ should become zero. Therefore, the next state of the system only relies on a portion of the information from its environment that is shared with the system at its current state. Incorporating the definition of the closed system into Equation \ref{eq:eq11}, and utilizing the fact that information cannot be negative, we have the following inequality:

\begin{equation}
    \label{inequality}
    I(S^C_{n+1} ; S^C_n) \geq I(S^C_{n+1} ; S^C_n|E^O_n)
\end{equation}
    
To further decompose this inequality, we change information into entropy using Equation \ref{eq:mutual-info-entropy}. We replace $I(S^C_{n+1} ; S^C_n)$ with  $H(S^C_{n+1}) + H(S^C_n) - H(S^C_{n+1}, S^C_n)$. We also replace $I (S^C_{n+1} ; S^C_n|E^O_n)$ with $H(S^C_{n+1}) + H(S^C_n|E^O_n) - H(S^C_{n+1}, S^C_n|E^O_n)$ and incorporate them into Equation \ref{inequality}. Therefore, we have:

\begin{dmath}
    \label{entropy_inequality}
    H(S^C_{n+1}) + H(S^C_n) - H(S^C_{n+1}, S^C_n) \geq H(S^C_{n+1}) + H(S^C_n|E^O_n) - H(S^C_{n+1}, S^C_n|E^O_n)
\end{dmath}

Based on the chain rule, we can provide relationships between mutual and conditional entropy for $S^C_n$ and $E^O_n$. In other words, we have $H(E^O_n) +H(S^C_n |E^O_n) = H(S^C_n)+H(E^O_n|S^C_n) = H(S^C_n, E^O_n)$. The relation between mutual and conditional entropy is similar to that of probability. We can achieve mutual entropy of two systems using the summation of the entropy of one system and the entropy of the second system given the first system. Using simple algebra, if $H(S^C_n)$ is replaced with $H(E^O_n) + H(S^C_n|E^O_n) - H(E^O_n|S^C_n)$ and substitute it into Equation \ref{entropy_inequality}, we have:

\begin{dmath}
\label{eq:minimu-mutual-info-derived}
    H(E^O_n)-H(E^O_n|S^C_n) \geq H(S^C_{n+1}, S^C_n) - H(S^C_{n+1}, S^C_n|E^O_n)
\end{dmath}

Equation \ref{eq:minimu-mutual-info-derived} is the entropy representation of Equation \ref{inequality}. Now to simplify Equation \ref{eq:minimu-mutual-info-derived}, we utilize Equation \ref{eq:eq6} where we have $H(E^O_n|S^C_n) = H(S^C_n,E^O_n) -H(S^C_n) $. Replacing $H(E^O_n|S^C_n)$ with $H(S^C_n,E^O_n) -H(S^C_n)$, and substituting $H(S^C_n) + H(E^O_n) - H(S^C_n,E^O_n)$ with $I(S^C_n;E^O_n)$ in Equation \ref{eq:minimu-mutual-info-derived}, we will get the following inequality for minimum mutual information between $S^C_n$ and $E^O_n$.\\

\begin{theorem}[\textbf{Inequality for mutual information in closure}]
\label{eq:minimu-mutual-info}
   \[ I(S^C_n; E^O_n) \geq H(S^C_{n+1}, S^C_n) - H(S^C_{n+1}, S^C_n|E^O_n)\]
\end{theorem}

\vspace{4mm}
Theorem \ref{eq:minimu-mutual-info} provides the relation for the level of mutual information being presented in the boundary of an informationally closed system. To maintain closure at state $n+1$, the output of the closed system to the environment at state $n$ (i.e., the amount of information transmitted from the system to its environment) should follow Theorem \ref{eq:minimu-mutual-info}. This inequality shows that the capacity of the channel between environment $E^O_n$ and $S^C_n$ should be more than a specific minimum. This lower bound is also dependent on the entropy of the next state of the system; $S^C_{n+1}$. Theorem \ref{eq:minimu-mutual-info} has an edge case scenario where mutual information of $S^C_n$ and $E^O_n$ becomes zero while being informationally closed. The zero mutual information in Theorem \ref{eq:minimu-mutual-info} indicates the two entropy $H(S^C_{n+1}, S^C_n)$ and $H(S^C_{n+1}, S^C_n|E^O_n)$ are equal. This equality means that the system does not send or get any relevant information to or from the current state of the environment; thus the system becomes functionally closed. 

The inequality in Theorem \ref{eq:minimu-mutual-info} ensures that the information transmitted from the system at state $n$ to the system at state $n+1$; $I(S^C_n; S^C_{n+1})$; would be maximal to confirm the closure of information in the closed system boundary. The level of mutual information between the current state of the closed system and its environment will be determined by the difference between the mutual entropy of the system at state $n$ and at the next state $n+1$ as well as the mutual information between the system at the two states $n+1$ and $n$ given $E^O_n$. \\

\textit{This dependency on mutual entropy of $S^C_{n+1}$ and $S^C_n$ indicates that in order to enable informationally closed systems engineering, one needs to capture the space of unknown states for the system in the future given the information of its current states, and current environment}.\\

Therefore, informational closure in SE is tightly interdependent with the ability to identify incomplete information and the states of the closed system in the future. In other words, by employing informational closure in SE, we don't need to have complete information of the system of interest's future states; ${S^0}_{n+1}$. There will be no need to provide a capability to predict 100\% of the the system of interest's states in advance. Rather, it relies on identification of what will be unknown in the next state of the system $H(S^C_{n+1})$ given the current level of information of $S^C_n$ and $S^C_n|E^O_n$. The fact that complete predictability is not a required feature for informational closure makes this concept suitable in the process of bounding intelligence. Intelligent systems may encounter new and/or unpredicted conditions and they are expected to learn, survive, and meet their goals when encountering such conditions. The closedness of information allows systems engineers to bound intelligent systems' environment and create a closed intelligence within the closed system's boundaries that emerges from the high coupling between the system of interest and part of its environment, $E^I$. Using informational closure, systems engineers decide what part of the environment needs to be coupled and modeled with the systems of interest\cite{shadab2022closed}. 

Systems engineers require to recognize and interpret what part of information should be engineered as mutual information between the closed system and its environment. Thus far, we concluded that the level of mutual information between the system and its environment should be present at a certain minimum level for realization of a closed system. However, mutual information, due to its subjective nature, should be determined as an important design choice in different engineering applications. Therefore; to understand the influence of design constraints on the level of mutual information, its upper bound can be captured by specifying a cost function. This cost function can be derived by other properties of the system of interest such as safety, cost, etc. As a result, systems engineers need to do trade-off analysis between $H(S^C_{n+1}, S^C_n)$ and $H(S^C_{n+1}, S^C_n|E^O_n)$ to achieve the inequality in Theorem \ref{eq:minimu-mutual-info}. We can maximize $H(S^C_{n+1}, S^C_n)$ up to a threshold $\delta$ by minimizing $H(S^C_{n+1}, S^C_n|E^O_n)$, where:

\begin{equation}
\begin{split}
    H(S^C_{n+1}, S^C_n) - H(S^C_{n+1}, S^C_n|E^O_n) \leq \\
    I(S^C_{n}; E^O_{n}) \leq \delta
\end{split}
\end{equation}

We argue that $\delta$ would be a design parameter that would depend on several factors. Mutual information between the closed system and its environment helps improve the level of predictability of environment inputs to the system. This factor is a critical measure to determine the upper bound of the mutual information between the closed system and its environment. increasing mutual information between the closed system and its environment has a direct relation with increasing the cost of design and engineering of the system of interest as it adds more complexity and needs more resources to provide this communication channel between the environment and the system. As a result, cost can be considered as constraint to the upper bound of mutual information.

\section{Conclusion}
In this paper, we asserted formalism for two types of closure; functional, and informational; and compared them with each other. We derived constraints to achieve each type of closure in systems. We showed that functional closure is a special form of informational closure. We also argued that intelligent systems can be designed and engineered at multiple levels of abstraction. Some levels are informationally closed; while others are informationally and functionally open, etc. We provided foundation on how the open and closed systems concept can be utilized to engineer different levels of abstractions for such systems. Finally, we elaborated on how we can exert constraints for the level of mutual information required between the next state of the system and the current states of the system and its environment to enable informational closure.

\bibliographystyle{IEEEtranIES}
\bibliography{ref}

\end{document}